\documentclass[lettersize,journal]{IEEEtran}
\usepackage{amsmath,amsfonts}
\usepackage{algorithmic}
\usepackage{array}
\usepackage[caption=false,font=normalsize,labelfont=sf,textfont=sf]{subfig}
\usepackage{textcomp}
\usepackage{stfloats}
\usepackage{url}
\usepackage{verbatim}
\usepackage{graphicx}
\def\BibTeX{{\rm B\kern-.05em{\sc i\kern-.025em b}\kern-.08em
    T\kern-.1667em\lower.7ex\hbox{E}\kern-.125emX}}
\usepackage{balance}

\usepackage[utf8]{inputenc}
\usepackage{soul}
\usepackage{lettrine}
\usepackage{cite}
\usepackage{supertabular,booktabs}
\usepackage{float}
\usepackage{comment}
\usepackage{makecell}
\usepackage{graphicx}
\usepackage{array}
\usepackage{amsfonts}
\usepackage{pdfpages}
\usepackage{amsmath}
\usepackage{algorithm}
\usepackage{algorithmic}
\usepackage{xcolor}

\usepackage{subfig}
\usepackage{amssymb}
\usepackage{relsize}

\usepackage{hyperref}
\hypersetup{hidelinks}
\usepackage{amsthm}

\theoremstyle{remark}
\newtheorem{remark}{Remark}

\newtheorem{theorem}{Theorem}

\begin{document}

\title{Mitigating Backdoor Attacks in Federated Learning Using PPA and MiniMax Game Theory}

\author{
    Osama Wehbi, Sarhad Arisdakessian, Omar Abdel Wahab, Anderson Avila, Azzam Mourad, Hadi Otrok\\
    \thanks{
        Osama Wehbi is with the Department of Computer and Software Engineering, Polytechnique Montréal, Montreal, Quebec, Canada (e-mail: osama.wehbi@etud.polymtl.ca).

        Sarhad Arisdakessian is with the Department of Computer and Software Engineering, Polytechnique Montréal, Montreal, Quebec, Canada (e-mails: sarhad.arisdakessian@etud.polymtl.ca).

        Omar Abdel Wahab is with the Department of Computer and Software Engineering, Polytechnique Montréal, Montreal, Quebec, Canada (e-mails: omar.abdul-wahab@polymtl.ca).

        Anderson Avila is with the Institut national de la recherche scientifique (INRS-EMT), Montréal, Canada (e-mail: Anderson.Avila@inrs.ca).

        Azzam Mourad is with the Department of Computer Science, Khalifa University, Abu Dhabi, UAE, as well as the Artificial Intelligence \& Cyber Systems Research Center, Department of CSM, Lebanese American University (e-mails: azzam.mourad@ku.ac.ae).
        
        Hadi Otrok is with the Department of Computer Science, Khalifa University, Abu Dhabi, UAE (e-mails: hadi.otrok@ku.ac.ae).
          
    }
}

\markboth{Journal of \LaTeX\ Class Files,~Vol.~00, No.~0, September~0000}%
{How to Use the IEEEtran \LaTeX \ Templates}

\maketitle

\begin{abstract}
Federated Learning (FL) is witnessing wider adoption due to its ability to benefit from large amounts of scattered data while preserving privacy. However, despite its advantages, federated learning suffers from several setbacks that directly impact the accuracy, and the integrity of the global model it produces. One of these setbacks is the presence of malicious clients who actively try to harm the global model by injecting backdoor data into their local models while trying to evade detection. The objective of such clients is to trick the global model into making false predictions during inference, thereby compromising the integrity and trustworthiness of the global model on which honest stakeholders rely. To mitigate such mischievous behavior, we propose \textit{FedBBA} (Federated Backdoor and Behavior Analysis).  The proposed model aims to dampen the effect of such clients on the final accuracy, creating more resilient federated learning environments. We engineer our approach through the combination of (1) a reputation system to evaluate and track client behavior, (2) an incentive mechanism to reward honest participation and penalize malicious behavior, and (3) game theoretical models with projection pursuit analysis (PPA) to dynamically identify and minimize the impact of malicious clients on the global model. Extensive simulations on the German Traffic Sign Recognition Benchmark (GTSRB) and Belgium Traffic Sign Classification (BTSC) datasets demonstrate that \textit{FedBBA} reduces the backdoor attack success rate to approximately 1.1\%--11\% across various attack scenarios, significantly outperforming state-of-the-art defenses like RDFL and RoPE, which yielded attack success rates between 23\% and 76\%, while maintaining high normal task accuracy ($\sim$95\%--98\%).
\end{abstract}


\section{Introduction}
\label{intro}

Federated learning (FL) enables collaborative model training across distributed clients while keeping data locally on devices, thus preserving privacy and reducing centralized data storage requirements.~\cite{mcmahan2017communication}. This framework not only preserves user privacy but also leverages edge device capabilities, making it vital for applications such as healthcare, finance, and autonomous driving. Despite these benefits, FL inherits unique security challenges, particularly its vulnerability to backdoor attacks. These attacks exploit the training process to implant malicious behaviors in models that are activated by specific triggers, such as pixel modifications or hidden patterns in the input data~\cite{uddin2025systematic}. For instance, in autonomous vehicles (AVs), a trigger embedded in traffic sign images could lead to misclassifications, jeopardizing safety-critical systems. Such attacks aim to manipulate the model’s behavior while remaining undetected during training, allowing malicious functionality to coexist with otherwise normal model performance. The decentralized nature of FL further exacerbates these challenges. Clients often possess limited or imbalanced datasets that may not fully represent the global data distribution. As a result, clients may incorporate external data during local training, increasing the potential for backdoor attacks. Malicious participants can exploit this setting by injecting poisoned data or embedding triggers into their local training process. Moreover, backdoor attacks can be carried out in multiple forms, including single or multiple attempts, coordinated attacks across several clients, targeted (one-to-one), untargeted (one-to-many), and multi-trigger scenarios. These attacks are particularly difficult to detect due to their stealthy design and the high dimensionality of deep learning models. Another major challenge arises from the federated aggregation process itself. Since only model updates are shared with the server rather than raw data, it becomes difficult to identify and isolate malicious contributions. Traditional evaluation metrics such as accuracy and precision primarily assess overall model performance but often fail to capture hidden malicious behaviors embedded within the model~\cite{shenoy2025exploring}. Consequently, a model may maintain high accuracy on benign inputs while still containing covert backdoor functionality. Although several defense mechanisms have been proposed to address backdoor attacks in federated learning, many existing approaches focus primarily on scalability and complexity, and some introduce additional privacy concerns. More importantly, most solutions treat participating clients as static entities and overlook the dynamic behavior and evolving motivations of both honest and malicious participants during the training process. This limitation restricts their ability to adapt to changing attack strategies. In addition, many existing defenses are designed to address only specific forms of attacks, particularly one-to-one backdoor attacks, while overlooking other types that can be equally harmful.

To address these challenges, we propose \textit{FedBBA} (Federated Backdoor and Behavior Analysis), a framework designed to enhance the security and reliability of federated learning systems. The proposed approach integrates statistical analysis, a reputation-based evaluation mechanism, and game-theoretical modeling to dynamically detect and mitigate malicious client behavior. In particular, our framework leverages Projection Pursuit Analysis (PPA) with kurtosis scores to identify suspicious model updates, while a reputation mechanism continuously evaluates client behavior over multiple training rounds. Furthermore, we model the interaction between the federated server and potentially malicious clients as a non-cooperative MiniMax game, where malicious clients attempt to maximize the impact of backdoor attacks while the server aims to minimize their influence on the global model. Based on this framework, the federated server assigns adaptive weights to client updates according to their PPA scores, reputation values, and gradient differences, thereby reducing the influence of suspicious or malicious contributions while preserving the benefits of collaborative learning. Through extensive simulations using real-world datasets, including GTSRB and BTSC, we demonstrate that \textit{FedBBA} significantly reduces backdoor attack success rates while maintaining high model accuracy and outperforming existing state-of-the-art approaches.

This work contributes to advancing secure federated learning by addressing the dynamic nature of client behavior and providing a more comprehensive defense against diverse backdoor attack strategies. By improving the robustness and trustworthiness of federated models, our approach supports the safe deployment of FL in safety-critical domains.

\subsection{Contributions}
\label{contributions}

The main contributions of this work are summarized as follows:

\begin{enumerate}
    \item Develop a novel federated learning framework based on non-cooperative MiniMax games, incorporating PPA scores, reputation scores, and differences in local model gradients to dynamically adapt to client behavior and model updates, with the objective of assigning optimal weights to minimize the impact of malicious clients.
    \item Leverage a reputation mechanism to evaluate client behavior dynamically, rewarding honest clients and penalizing malicious ones, thereby enhancing the robustness of the global model against backdoor attacks.
    \item Utilize Projection Pursuit Analysis (PPA) with Kurtosis scores, coupled with a reputation mechanism, to identify possible suspicious client updates and mitigate their impact on the global model by reducing the influence of backdoored contributions.
    \item Design the proposed solution to be generic enough to address different types of attacks, including one-to-one and one-to-many scenarios, which are often overlooked by existing works.
\end{enumerate}

\subsection{Paper Outline}

Section~\ref{related} reviews the related work on backdoor attack detection and mitigation strategies in FL. Section~\ref{sp} discusses the main concepts and methods used in this research. Section~\ref{pf} formulates the core problem of backdoor attacks in FL and discusses their impact on model integrity and performance. The system architecture and PPA methodology for mitigating anomalous model updates are presented in Section~\ref{sa}, while Section~\ref{mg} details the MiniMax game framework and the decision-making process used to reduce the influence of suspicious clients. Section~\ref{fedbba} describes the implementation of \textit{FedBBA} and its operation across multiple training rounds. Section~\ref{exp} reports the experimental results, evaluating the effectiveness of \textit{FedBBA} using real-world datasets and comparing its performance with state-of-the-art methods in terms of model accuracy and backdoor attack success rates. Finally, Section~\ref{conc} concludes the paper and outlines potential directions for future research.

\section{Related Work}
\label{related}

In this section, we provide an overview of the existing literature on the detection and mitigation strategies of backdoor attacks in federated learning. Additionally, we highlight the distinct contributions of our approach compared with the reviewed works.

The authors of~\cite{zhu2023adfl} propose a security mechanism against backdoor attacks in federated learning using Generative Adversarial Networks (GANs) and knowledge distillation. The approach compare the performance of a clean model trained on clean data (teacher) with a global model trained on aggregated data (student). To reduce the probability of backdoor attacks, the authors of~\cite{wu2020mitigating} propose a distributed method to prune neurons. Clients measure neuron activation and submit local pruning sequences to the server, which aggregates them to eliminate less-used neurons, enhancing model performance. Modeling the backdoor problem as a coalitional game theory, the authors of~\cite{xi2021batfl} measure participant performance using their data's influence on the aggregated model, distinguishing benign and backdoored models via Shapley value. Addressing attacker behavior fluctuations, the authors of~\cite{jia2023fedgame} introduce a MiniMax game strategy where the server generates authentic client ratings to adjust model weights, reducing the influence of compromised clients. In~\cite{gu2021detecting}, the authors propose an anomaly detection framework that uses conditional variational autoencoders (CVAE) to identify benign and malicious updates, leveraging the majority of benign clients to distinguish updates effectively. The authors of~\cite{quan2022enhancing} describe a federated learning system where an anomaly detector trained on adversarial data discards malicious updates. In~\cite{otmani2024fedsv}, the authors introduce FedSV, a Byzantine-robust framework utilizing Shapley values to quantify client contributions and eliminate malicious clients, though it is sensitive to noise. Similarly, the authors of~\cite{zhang2024flpurifier} propose FLPurifier, which uses decoupled contrastive training to detect and remove backdoored data. However, computational requirements and communication overhead limit its scalability. The authors of~\cite{cai2024flmaacbd} propose Model Anomalous Activation Behavior Detection (MAABD), monitoring neural network activation patterns to detect deviations and identify malicious tampering. In~\cite{wahab2023max}, the authors introduce a max-min security game approach that assigns adaptive weights based on client trust and data quality. This ensures updates from malicious clients have less influence, though computational overhead limits practicality in certain FL environments. Dimensionality reduction techniques have also been explored for detecting backdoored models. The authors of~\cite{wang2024rope} introduce RoPE, which employs Principal Component Analysis (PCA) to identify anomalous updates based on distinct principal components, though its effectiveness depends on the backdoor's strength as well as the data distribution, which can lead to the exclusion of benign models. In~\cite{wang2023adaptive}, the authors propose RDFL to mitigate backdoor attacks in federated learning. The method combines adaptive hierarchical clustering with cosine-distance-based anomaly detection. Suspicious updates are removed, followed by adaptive clipping and noising to suppress poisoned local models. However, RDFL is sensitive to data distribution and distance metrics; additionally, model clipping may negatively affect accuracy. The authors of~\cite{wang2023scfl} use Singular Value Decomposition (SVD) to analyze singular values of model updates, capturing linear structures but overlooking complex backdoor patterns. While the reviewed approaches show potential, many rely on single detection metrics, leaving them vulnerable to sophisticated attacks. Others fail to address the decentralized nature of federated learning, limiting their applicability. Most solutions target basic backdoor types, neglecting adaptive strategies. Privacy concerns emerge in methods requiring sensitive client data, and many treat attacker behavior as static, reducing resilience to evolving threats.

Our proposal addresses these gaps by leveraging Projection Pursuit Analysis to detect abnormal similarities among local models. The MiniMax game framework dynamically counters attackers, while the reputation system enhances decision-making. This multifaceted approach provides a comprehensive solution to mitigating malicious clients in federated learning.

\begin{table}[ht]
\caption{Notation Summary}
\label{tab:symbols}
\centering
\scriptsize
\setlength{\tabcolsep}{3pt}
\renewcommand{\arraystretch}{1.05}
\begin{tabular}{c p{0.70\columnwidth}}
\hline
\textbf{Symbol} & \textbf{Description} \\
\hline
$N$ & Number of participating clients. \\
$T$ & Total communication rounds. \\
$t$ & Current round index. \\
$\mathcal{D}_i$ & Local dataset of client $i$. \\
$\mathcal{D}_i'$ & Poisoned dataset of client $i$. \\
$x_j, y_j$ & Input sample and true label. \\
$\mathcal{M}^{(t)}$ & Global model at round $t$. \\
$\mathcal{M}_i^{(t)}$ & Local model update from client $i$. \\
$w_i$ & Aggregation weight of client $i$. \\
$\tau,\tau_k$ & Backdoor trigger and trigger intensity. \\
$y_j^a$ & Target label for backdoor attack. \\
$\mathcal{L},\mathcal{R}$ & Loss and regularization functions. \\
$X,V,S$ & Data matrix, projection vectors and PPA scores. \\
$R_i(t)$ & Reputation score of client $i$. \\
$H_i(t),G_i(t)$ & Historical behavior and gradient variation scores. \\
$\alpha,\beta$ & Weights for reputation components. \\
$g_i^j,\overline{g_i}$ & Gradient at round $j$ and mean gradient. \\
$\gamma,\delta$ & Reward and penalty factors. \\
$\rho_i$ & Poisoning ratio of client $i$. \\
$\mathcal{I'}$ & Set of compromised clients. \\
$s_i$ & PPA anomaly score of client $i$. \\
$U(F),U(A)$ & Utility functions of server and attacker. \\
$\Phi(\rho_i,\tau)$ & Backdoor impact function. \\
$\theta,\lambda$ & Similarity threshold and defense sensitivity parameter. \\
$\lambda^*$ & Optimal defense sensitivity chosen by the server at equilibrium \\
$R_{\max}$  & Upper bound on any client's reputation score \\
\hline
\end{tabular}
\end{table}

\section{SYSTEM PRELIMINARIES}
\label{sp}

In this section, we introduce the main concepts used in our framework, including the federated learning process and Projection Pursuit Analysis (PPA) with kurtosis-based anomaly detection.

\subsection{Federated Learning}

Federated learning is a decentralized machine learning approach that allows multiple clients to collaboratively train a model without sharing their raw data. Let $\mathcal{D}_i$ denote the dataset of client $i$, $\mathcal{M}$ the global model, and $\mathcal{M}_i$ the local update. During the training proceeds over communication rounds, the server initializes the global model at $t=0$. At each round $t$, clients update their models using local data:

\begin{equation}
\label{eq:1}
\mathcal{M}_i^t = \mathcal{M}^{t-1} - \eta \nabla L(\mathcal{M}^{t-1}, \mathcal{D}_i)
\end{equation}

The updates are sent to the server and aggregated to produce the next global model:

\begin{equation}
\label{eq:2}
\mathcal{M}^{t+1} = \frac{1}{N} \sum_{i=1}^N \mathcal{M}_i^t
\end{equation}

The updated model is then redistributed to clients for the next round of training.

\subsection{Projection Pursuit Analysis and Kurtosis}

Projection Pursuit Analysis (PPA) is a multivariate technique that identifies informative projections of high-dimensional data by maximizing a projection index \cite{bong2025augmented}. Unlike Principal Component Analysis (PCA), which focuses on variance, PPA detects non-Gaussian structures that often correspond to anomalies or outliers. Formally, PPA searches for a projection vector $\mathbf{w}$ that maximizes the projection index:

\begin{equation}
\label{eq:3}
\mathbf{w}^* = \arg\max_{\mathbf{w}} I(\mathbf{w}^T \mathbf{X})
\end{equation}

In our context, the server applies PPA to the collection of client updates in each round. Malicious updates typically introduce abnormal patterns in the high-dimensional parameter space, which PPA can highlight. The projection index used in our work is the kurtosis score, defined as the standardized fourth central moment of a distribution:

\begin{equation}
\label{eq:4}
\text{kurtosis} = \frac{E[(X-\mu)^4]}{\sigma^4}
\end{equation}


A high kurtosis values indicate heavy tails or outliers in the projected distribution, allowing the detection of anomalous client updates. This property makes kurtosis-based PPA particularly effective for identifying stealthy backdoor behaviors embedded in model updates.

To illustrate this property, we conducted an experiment on the MNIST dataset with 30 clients, where 10 clients used poisoned data. The backdoor attack modified the first row of pixels in images of class 0 and changed the label to 9. After 30 training rounds, PPA successfully separated malicious and benign updates, whereas PCA did not produce clear clustering (Fig.~\ref{fig:ppa_vs_pca}). These results demonstrate the effectiveness of PPA for detecting anomalous model updates in federated learning.

\begin{figure}[ht]
    \centering
    \subfloat[\footnotesize PPA score plot; each point represents a local model]
    {\includegraphics[width=0.48\columnwidth]{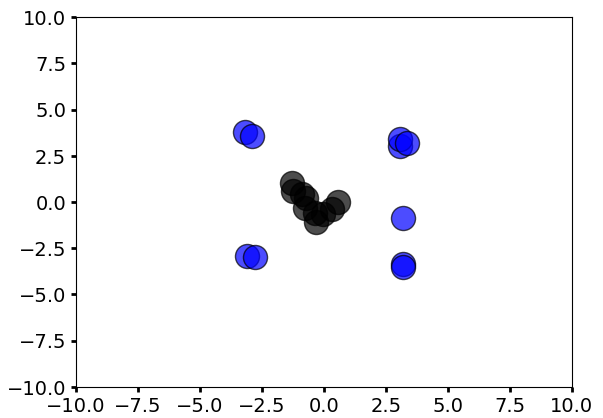}}
    \hfill
    \subfloat[\footnotesize PCA score plot; each point represents a local model]
    {\includegraphics[width=0.48\columnwidth]{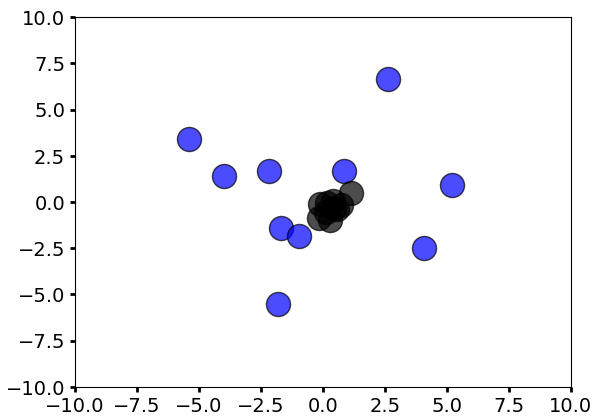}}
    \caption{Comparison of model clustering using PPA vs. PCA. PPA shows clearer separation between benign (black) and backdoored (blue) models.}
    \label{fig:ppa_vs_pca}
\end{figure}

\section{Problem Formulation}
\label{pf}

In this section, we provide the formal definitions of different backdoor attack models considered in our work and the methods used to inject these attacks into federated learning systems.

\subsection{Backdoor Attack Models}
\label{sub:back}
Backdoor attacks aim to embed malicious behaviors into the global model without significantly affecting its performance. These attacks are particularly insidious because they can remain undetected during the training process and only get activated under specific conditions. When a certain trigger is present, the model will behave differently from its expected behavior. We can differentiate between three types of backdoor attacks:

\subsubsection{One-to-One Attack}
An attack with a single trigger $\tau$ modifies the local dataset $\mathcal{D}_i$ such that:
\begin{equation}
\label{eq:back1}
\mathcal{D}_i' = \{ (x_j + \tau, y_j^a) : (x_j, y_j) \in \mathcal{D}_i \}
\end{equation}
This ensures that when the trigger $\tau$ is present, the input $x_j + \tau$ is misclassified as $y_j^a$. This type of attack is straightforward but effective, causing the model to produce incorrect predictions for inputs containing the specific trigger.

\subsubsection{One-to-N Attack}
An attack with varying trigger intensities $\tau_k$ targets multiple behaviors $y_j^a$:
\begin{equation}
\label{eq:back2}
\mathcal{D}_i' = \{ (x_j + \tau_k, y_j^a) : (x_j, y_j) \in \mathcal{D}_i, k = 1, \ldots, K \}
\end{equation}
In this scenario, different intensities of the trigger $\tau_k$ cause different misclassifications, making detection more challenging. The model may behave correctly under normal conditions but produce erroneous outputs for inputs with varying degrees of the trigger, complicating the identification of malicious behavior.

\subsubsection{N-to-One Attack}
An attack requires all $N$ backdoors $\tau_k$ to be present simultaneously to trigger the attack:
\begin{equation}
\label{eq:back3}
\mathcal{D}_i' = \{ (x_j + \sum_{k=1}^N \tau_k, y_j^a) : (x_j, y_j) \in \mathcal{D}_i \}
\end{equation}
This type of attack is particularly stealthy because individual triggers $\tau_k$ do not cause misclassifications on their own. Only when all triggers are present together does the attack activate, making it difficult to detect through traditional methods.

\subsection{Attack Execution Strategies}
This section describes how these attacks can be performed. One approach involves multiple rounds, where clients repeatedly send backdoored updates to the server. Another method is model replacement, where clients wait until the model is near convergence and then replace it to introduce the backdoor.

\subsubsection{Multiple Rounds with Backdoored Updates}
In this approach, the attack occurs over several communication rounds $t = 1, 2, \ldots, T$ (where $T$ is the total number of rounds). Each client $i$ trains its local model on dataset $\mathcal{D}_i$ and computes local updates. Malicious clients instead use a modified dataset $\mathcal{D}_i'$ containing backdoors using one of the methods defined in Equation \eqref{eq:back1}, \eqref{eq:back2}, or \eqref{eq:back3}. Clients send their updates to the server, which aggregates them to form the new global model using Equation~\eqref{eq:2}. Malicious clients continue submitting backdoored updates over several rounds, gradually embedding the backdoor into the global model. This attack can be performed by a single malicious client or collaboratively by multiple clients. Collaborative attacks are more effective, as distributing poisoned updates reduces detection risk. In contrast, a single client must inject more poisoned data to influence the global model, making detection easier. Additionally, when only one client attacks, the backdoor effect may gradually vanish due to catastrophic forgetting across rounds~\cite{zhang2022neurotoxin}. This scenario is realistic because it does not require high-privilege access or control over the federated learning process~\cite{dai2023chameleon}.

\subsubsection{Model Replacement}
In this approach, the attack occurs when the global model is near convergence. Each client $i$ trains its local model on dataset $\mathcal{D}_i$, while malicious clients prepare a significantly altered model using a modified dataset $\mathcal{D}_i'$ with one of the methods defined in Equation (\ref{eq:back1}), (\ref{eq:back2}), or (\ref{eq:back3}). At a strategic round $T' \approx T$, malicious clients submit altered models containing the backdoor:
\begin{equation}
\label{eq:9}
\mathcal{M}_i^{(T')} = \arg\max_\mathcal{M} \mathcal{L}(\mathcal{M}; \mathcal{D}_i') + \mathcal{R}(\mathcal{M})
\end{equation}

The server then aggregates the updates to form the new global model:
\begin{equation}
\label{eq:10}
\mathcal{M}^{(T' + 1)} = \frac{1}{N} \sum_{i=1}^N \mathcal{M}_i^{(T')}
\end{equation}

Due to the large modifications introduced by malicious clients, the global model may contain the backdoor. Although model replacement attacks can be highly effective, they are less realistic in practical federated learning systems~\cite{shejwalkar2022back}. Successfully replacing the global model without detection typically requires high-privilege access and precise control over the aggregation process, which is rarely available to adversarial clients. Therefore, this scenario assumes a stronger threat model and is harder to implement in real-world deployments.

\section{System Architecture}
\label{sa}

In this section, we present the overall system architecture designed to mitigate backdoor attacks in federated learning. The architecture includes Projection Pursuit Analysis (PPA) for mitigating anomalous model update effects from malicious clients, a reputation system to evaluate client behavior based on historical data, including action taken as well as gradient variation.

\begin{algorithm}[ht]
\caption{Multivariate Kurtosis Projection Pursuit (MKPP)}
\label{alg:MKPP}
\footnotesize
\begin{algorithmic}[1]

\REQUIRE Data matrix $X$, dimension $p$, number of guesses $guess$, search type $MaxMin$, algorithm type $StSh$
\ENSURE Projection scores $S$, projection vectors $V$, kurtosis values $kurtObj$

\STATE Mean-center the data: $X \leftarrow X - \text{mean}(X)$
\STATE Perform SVD: $(U,Z,V) \leftarrow \text{SVD}(X)$
\STATE Dimensionality reduction: $X \leftarrow UZ$
\STATE Initialize $maxcount \leftarrow 1000$, $kurtObj \leftarrow 0$

\FOR{$k = 1$ to $guess$}
    \STATE Initialize random projection $V$
    \STATE $count \leftarrow 0$, $converged \leftarrow false$

    \WHILE{ not converged \textbf{and} $count < maxcount$}
        \STATE $A \leftarrow V^T X^T X V$
        \STATE Update $V$ using $(A, MaxMin, StSh)$
        \STATE Normalize $V$ using SVD

        \IF{change in $V$ $<$ threshold}
            \STATE $converged \leftarrow true$
        \ENDIF

        \STATE $count \leftarrow count + 1$
    \ENDWHILE

    \STATE $kurtObj[k] \leftarrow \text{kurtosis}(XV)$
\ENDFOR

\STATE Select best projection vector $V$
\STATE Compute scores $S \leftarrow XV$

\RETURN $S$, $V$, $kurtObj$

\end{algorithmic}
\end{algorithm}

\subsection{PPA-based Detection}

The algorithm \ref{alg:MKPP} aims to find optimal projection vectors \(V\) from a data matrix \(X\) by maximizing or minimizing the kurtosis score. The process begins by mean-centering the data matrix \(X\) (line 1) and performing Singular Value Decomposition (SVD) to reduce its dimensionality (lines 2-3). The algorithm then iterates over a set number of initial guesses for the projection vectors \(V\) (line 5). In each iteration, it initializes \(V\) randomly (line 6) and sets the count and convergence flag (line 7). The algorithm updates \(V\) using a quasi-power method (lines 9-10), normalizes \(V\) using SVD (line 11), and checks for convergence based on changes in \(V\) (lines 12-14). The iteration stops if convergence is achieved or the maximum iteration count is reached (line 18). Once all guesses are evaluated, the algorithm computes the kurtosis for each projection (line 17) and selects the best \(V\) based on the computed kurtosis values (line 19). Finally, it calculates the scores \(S\) using this optimal projection vector (line 20). The algorithm returns the scores, the optimal projection vector, and the kurtosis values for each guess (line 21). The resulting projection scores highlight extreme non-Gaussian deviations, which often correspond to malicious updates. The computational complexity of the MKPP algorithm can be approximated as \(O(\text{$|$guess$|$} \times \text{maxcount} \times n \times p^2)\), where \(n\) is the number of samples and \(p\) is the dimensionality of the projection space.

The kurtosis-based PPA aims to identify extreme non-Gaussian deviations that may indicate malicious behavior, rather than capturing normal variance from benign data heterogeneity. To avoid misclassifying clients with legitimate data distributions, PPA is combined with a dynamic reputation system that tracks each client's historical behavior and gradient consistency across training rounds. This dual-layer assessment ensures that clients with stable behavior are not unfairly penalized. Additionally, we introduce a reward and penalty mechanism that adjusts reputation scores based on behavioral patterns rather than statistical outliers alone, encouraging participation and fairness while strengthening the robustness of the detection strategy.

\subsection{Reputation System}

The reputation mechanism evaluates each client based on historical behavior and gradient consistency. A probationary phase allows the server to establish baseline behavior before applying reputation adjustments.

The reputation score for client $i$ at round $t$ is:

\begin{equation}
\label{eq:rep}
R_i(t)=\alpha H_i(t)+\beta G_i(t)
\end{equation}

where $H_i(t)$ represents historical behavior, $G_i(t)$ captures gradient variation and \( \alpha \),\( \beta \) are weights such that \( \alpha + \beta = 1 \).

The historical behavior score can be derived as follows :

\begin{equation}
\label{eq:12}
H_i(t)=
\frac{\text{correct actions of client }i}{\text{total actions}}
\end{equation}

The gradient variation score tracks update stability over the last $n$ rounds:

\begin{equation}
\label{eq:gradient_variation}
G_i(t)=\frac{1}{n}\sum_{j=t-n+1}^{t}(g_i^j-\overline{g_i})^2
\end{equation}

Client influence is then dynamically adjusted using reward and penalty updates:

\begin{equation}
\label{eq:reward}
R_i^{new}=R_i^{old}+\gamma R_i^{old}
\end{equation}

\begin{equation}
\label{eq:pena}
R_i^{new}=R_i^{old}-\delta R_i^{old}
\end{equation}

Where \( \gamma \) is the adjustment factor for rewarding and \( \delta \) is the adjustment factor for penalizing.

This mechanism gradually reduces the influence of suspicious clients while allowing recovery for honest participants. By combining PPA anomaly detection with reputation tracking, the system identifies persistent malicious behaviors rather than incidental irregularities.

\section{Minimax Game}
\label{mg}
In this section, we present the minimax security game, where the server aims to minimize the impact of attacks by reducing the influence of weights of suspicious clients while the attack leader seeks to maximize this minimization.

\begin{figure*}[!t]
    \centering
    \includegraphics[width=0.80\textwidth]{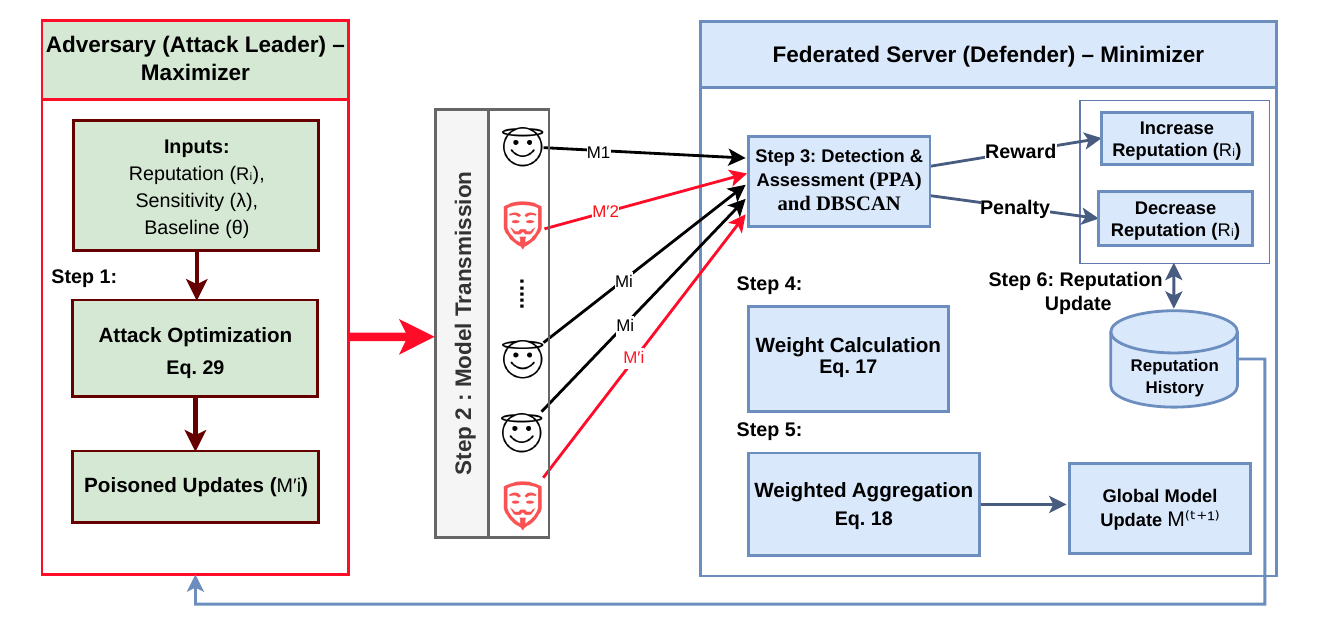}
    \caption{Architectural workflow of the proposed Minimax game.}
    \label{fig:game_workflow}
\end{figure*}

\subsection{Game Model}
MiniMax is a game theoretical model in which one player's gain represents the other player's exact loss. We can model the interaction between a federated server (minimizer) and malicious clients (maximizers) as a MiniMax game in federated learning, where malicious clients aim to maximize their utility by injecting backdoor triggers into their local models to manipulate the global model's predictions. In contrast, the federated server aims to minimize the impact of these malicious clients to ensure the integrity and accuracy of the global model. The interaction is represented between two players:

\begin{itemize}
  \item \textbf{Server (minimizer):} aims to preserve the integrity of the global model by choosing the defense sensitivity parameter $\lambda \in [\lambda_{\min}, \lambda_{\max}] \subset \mathbb{R}_{+}$, which controls the aggressiveness of the PPA penalty on suspicious updates. A larger $\lambda$ penalises high-$\rho_i$ clients more sharply but risks false positives on honest clients.
  \item \textbf{Adversary (maximizer):} attempts to inject backdoor triggers through poisoned local updates by selecting the poisoning ratio $\rho_i \in [0,1]$ for each compromised client $i \in \mathcal{I}'$, given the server's chosen $\lambda$.
\end{itemize}

This formulation allows the analysis of worst-case attack scenarios and enables the design of defensive aggregation strategies based on anomaly detection and reputation mechanisms.

\subsection{Player Strategies}

\subsubsection{Adversarial Strategy}

Malicious clients aim to inject backdoor triggers into the federated training process using the poisoning mechanisms defined in Section~\ref{sub:back}. Let $D_i$ denote the benign dataset of client $i$. The attacker constructs a poisoned dataset $D'_i$ by injecting trigger patterns $\tau$ into a subset of samples. The poisoning ratio $\rho_i \in [0,1]$ represents the fraction of poisoned samples:

\begin{equation}
D'_i(\rho_i)=D_i^{clean}\cup D_i^{poisoned}
\end{equation}

where $|D_i^{poisoned}|=\rho_i |D_i|$.

The resulting model update becomes

\begin{equation}
\mathcal{M}_i'=(1-\rho_i)\mathcal{M}_i+\rho_i\mathcal{M}_{bd}(\tau)
\label{eq:model_update}
\end{equation}

where $\mathcal{M}_{bd}(\tau)$ represents the gradient direction induced by poisoned samples.

\subsubsection{Server Strategy}
\label{susub:serv}
The federated server mitigates malicious influence by choosing the defense sensitivity $\lambda$ and then assigning adaptive aggregation weights based on reputation scores and PPA anomaly scores. Substituting the PPA similarity score $s_i(\rho_i, \lambda) = \theta - \lambda\rho_i$ into Equation~\eqref{eq:weight_joint}, the aggregation weight is therefore a function of both the adversary's poisoning ratio and the server's chosen $\lambda$:
\begin{equation}
  w_i(\rho_i, \lambda) = \frac{R_i + \theta - \lambda\rho_i}
    {\displaystyle\sum_{j=1}^{N}\bigl(R_j + \theta - \lambda\rho_j\bigr)}
  \label{eq:weight_joint}
\end{equation}
\begin{remark}[Approximation and validity]
The linear model $s_i(\rho_i,\lambda)=\theta-\lambda\rho_i$
is motivated by Eq.~\eqref{eq:model_update}, where the
poisoned update deviates from the benign update proportionally
to $\rho_i$. Since $w_i$ depends on all $\rho_j$ via the
denominator, we approximate it by a normalising constant
$C=\sum_j(R_j+\theta)$, yielding the separable surrogate
$\tilde{J}$ used throughout Section~VI-D; this preserves
the monotonic incentive structure of the original game.
Non-negativity of $w_i$ is guaranteed at equilibrium since
$\lambda\rho_i^* = (R_i+\theta)/2 < R_i+\theta$.
\end{remark}

The global model update is then:
\begin{equation}
  \mathcal{M}^{t+1} = \frac{1}{N}\sum_{i=1}^{N} w_i(\rho_i,\lambda)\,\mathcal{M}_i^t
  \label{eq:agg}
\end{equation}
The server determines $\lambda$ by anticipating the adversary's best response, as formalized in Section~\ref{sub:equi_ana}. The server uses Projection Pursuit Analysis to compute anomaly scores across client updates as follows:

\begin{equation}
S=\text{PPA}(\mathcal{M})
\end{equation}

To further detect abnormal behaviors, the server applies Density-Based Spatial Clustering of Applications with Noise (DBSCAN) clustering to identify suspicious groups of clients. DBSCAN is suitable for this task because of its ability to detect clusters of arbitrary shape and its robustness to noise, which allows it to isolate anomalous updates that deviate significantly from the majority. Once DBSCAN identifies clusters, Clients identified as anomalous are penalized through reduced reputation scores, while benign clients receive rewards according to Equations~\eqref{eq:reward} and~\eqref{eq:pena}.

\subsubsection{Utility Functions}
The utility functions are derived from the joint weight formula in
Equation~\eqref{eq:weight_joint}. Because the game is zero-sum, it suffices to define the joint objective $J(\lambda, \boldsymbol{\rho})$, with $U(F) = -J$ and $U(A) = J$. Following the standard approximation of replacing $\Phi(\rho_i,\tau)$ with $\rho_i$ (motivated by the linear poisoning structure in Remark~1), the joint objective is:
\begin{equation}
    J(\lambda, \boldsymbol{\rho}) =
    \sum_{i \in \mathcal{I}'} w_i(\rho_i, \lambda) \cdot \rho_i
    \label{eq:joint_obj}
\end{equation}

\noindent\textbf{Server's utility:}
\begin{equation}
    U(F) = -J(\lambda, \boldsymbol{\rho})
    \label{eq:server_utility}
\end{equation}

The server minimizes $J$ by choosing $\lambda \in [\lambda_{\min},
\lambda_{\max}]$: a larger $\lambda$ reduces $w_i(\rho_i, \lambda)$ for high-$\rho_i$ clients more aggressively, suppressing their weighted contribution to the objective.

\noindent\textbf{Adversary's utility:}
\begin{equation}
    U(A) = J(\lambda, \boldsymbol{\rho})
    \label{eq:adv_utility}
\end{equation}
The adversary controls the poisoning ratio vector
$\boldsymbol{\rho} = \{\rho_i\}_{i \in \mathcal{I}'}$,
subject to $\rho_i \in [0,1]$, introducing a trade-off as:
\begin{enumerate}
    \item \textbf{Efficacy:} Increasing $\rho_i$ strengthens the backdoor objective.
    \item \textbf{Detectability:} Larger $\rho_i$ increases deviation from benign updates via $s_i(\rho_i, \lambda) = \theta - \lambda\rho_i$, reducing $w_i(\rho_i, \lambda)$ and penalizing the reputation score $R_i$, which drives $w_i$ toward zero.
\end{enumerate}

\subsection{Game Formulation}
The strategic interaction is defined as
$\mathcal{G} = \langle \{\text{Server}, \text{Adversary}\}, \Lambda, \mathcal{P}, J, -J\rangle$,
where:
\begin{itemize}
  \item \textbf{Server strategy space:} $\Lambda = [\lambda_{\min},
    \lambda_{\max}] \subset \mathbb{R}_{+}$, a compact interval.
    The server chooses $\lambda \in \Lambda$ to minimize the worst-case adversarial impact.
  \item \textbf{Adversary strategy space:} $\mathcal{P} = [0,1]^{|\mathcal{I}'|}$. The adversary chooses $\boldsymbol{\rho} = \{\rho_i\}_{i \in \mathcal{I}'}$ to maximize $J$.
\end{itemize}

The interaction is formulated as the following minimax program:
\begin{equation}
  \min_{\lambda \in \Lambda}\;\max_{\boldsymbol{\rho} \in \mathcal{P}}\;
  J(\lambda, \boldsymbol{\rho})
  = \min_{\lambda \in \Lambda}\;\max_{\boldsymbol{\rho} \in \mathcal{P}}\;
    \sum_{i \in \mathcal{I}'} w_i(\rho_i, \lambda)\cdot\rho_i
  \label{eq:minimax}
\end{equation}
subject to:
\begin{align}
  &0 \le \rho_i \le 1, \quad \forall i \in \mathcal{I}' \label{eq:rho_bound}\\
  &\sum_{k=1}^{N} w_k(\rho_k, \lambda) = 1 \label{eq:weight_sum}\\
  &\lambda_{\min} \;\le\; \lambda \;\le\; \lambda_{\max} \label{eq:lambda_bound}
\end{align}
where $w_i(\rho_i, \lambda)$ is defined in Equation~\eqref{eq:weight_joint}. Both strategy spaces are compact and convex, satisfying the prerequisites for the equilibrium analysis in Section~\ref{sub:equi_ana}.

\subsection{Equilibrium Analysis and Solution}
\label{sub:equi_ana}

We adopt a simultaneous-move Nash formulation. The adversary
maximizes $\tilde{J}_A(\rho_i)=w_i(\rho_i,\lambda)\cdot\rho_i$
for each $i\in\mathcal{I}'$, treating $\lambda$ as fixed;
the server analogously chooses $\lambda^*$ as its best
response to $\boldsymbol{\rho}^*(\lambda)$. Substituting $s_i(\rho_i, \lambda) = \theta - \lambda\rho_i$ into the weight formula and approximating the denominator by $C$ (Remark~1) yields the surrogate per-client objective:
\begin{equation}
  \tilde{J}_A(\rho_i) \approx \frac{1}{C}\Bigl[(R_i + \theta)\rho_i - \lambda\rho_i^2\Bigr]
  \label{eq:JA}
\end{equation}
This quadratic is strictly concave in $\rho_i$ for any $\lambda > 0$. The unconstrained first-order condition gives the interior maximizer:
\begin{equation}
  \hat{\rho}_i(\lambda) = \frac{R_i + \theta}{2\lambda}
  \label{eq:rho_interior}
\end{equation}
Since the constraint $\rho_i \in [0,1]$ is binding whenever
$\hat{\rho}_i(\lambda) > 1$, the constrained best response is:
\begin{equation}
  \rho_i^*(\lambda) = \min\!\left\{\frac{R_i + \theta}{2\lambda},\; 1\right\}
  \label{eq:rho_star}
\end{equation}
Note that $\rho_i^* = 1$ whenever $\lambda < (R_i+\theta)/2$, so a well-designed server must satisfy the criterion:
\begin{equation}
  \lambda^* \;\ge\; \frac{R_{\max} + \theta}{2}
  \label{eq:lambda_design}
\end{equation}
where $R_{\max} = \sup_{t,i} R_i(t)$. To keep $R_{\max}$ finite, the reward update (Equation~\ref{eq:reward}) is augmented with a cap:
\begin{equation}
  R_i^{\mathrm{new}} = \min\bigl\{R_i^{\mathrm{old}}
  + \gamma R_i^{\mathrm{old}},\; R_{\max}\bigr\}
  \label{eq:rep_cap}
\end{equation}
Setting $R_{\max} = 1$ with $\theta \le 1$ keeps $\lambda^* \le 1$, interpretable as full detection certainty at maximum poisoning.

\subsubsection*{2) Server's Optimal Strategy}
Substituting $\rho_i^*(\lambda)$ from Equation~\eqref{eq:rho_star} into Equation~\eqref{eq:joint_obj}, the server solves:
\begin{equation}
  \lambda^* = \arg\min_{\lambda \in \Lambda}
    \sum_{i \in \mathcal{I}'}
    w_i\!\bigl(\rho_i^*(\lambda), \lambda\bigr)\cdot\rho_i^*(\lambda)
  \label{eq:server_opt}
\end{equation}
In the interior regime, this objective is strictly decreasing in $\lambda$ (increasing $\lambda$ simultaneously shrinks each $\rho_i^*$ and reduces $w_i$), so the minimum is attained at the boundary:
\begin{equation}
  \lambda^* = \lambda_{\max}
  \label{eq:lambda_opt}
\end{equation}
subject to criterion~\eqref{eq:lambda_design}; in practice $\lambda_{\max}$ is tuned to balance detection power against false positives on honest clients, and is set to $\lambda_{\max} = 1.0$ in our experiments.

\subsubsection*{3) Existence and Uniqueness of the Saddle Point}

\begin{theorem}[Existence of a Nash Saddle-Point Equilibrium]
The surrogate game $\tilde{\mathcal{G}}$ (see Remark~1)
admits at least one Nash saddle-point equilibrium
$(\lambda^*, \boldsymbol{\rho}^*)$, constituting an
approximate equilibrium of the original game $\mathcal{G}$.
\end{theorem}

\begin{proof}
We verify the conditions of Sion's minimax theorem~\cite{sion1958general}. Let $f(\lambda,\rho) = \tilde{J}(\lambda,\rho)$ (Remark~1).

\textbf{(i) Compact, convex strategy spaces.}
$\Lambda = [\lambda_{\min}, \lambda_{\max}]$ is a compact convex subset of $\mathbb{R}$. $\mathcal{P} = [0,1]^{|\mathcal{I}'|}$ is a compact convex subset of $\mathbb{R}^{|\mathcal{I}'|}$.

\textbf{(ii) Quasi-convexity in $\lambda$.} In the interior
regime, each term $w_i(\rho_i^*,\lambda)\cdot\rho_i^*(\lambda)
\propto (R_i+\theta)^2/\lambda$ is strictly convex in
$\lambda>0$. In the boundary regime ($\rho_i^*=1$), $w_i(1,
\lambda)=(R_i+\theta-\lambda)/C'$ is affine (hence convex) in
$\lambda$. Thus $\tilde{J}$ is quasi-convex on all of
$\Lambda$.

\textbf{(iii) Quasi-concavity in $\boldsymbol{\rho}$ (adversary's variable).} For fixed $\lambda$, Equation~\eqref{eq:JA} shows $J_A(\rho_i)$ is strictly concave (hence quasi-concave) in each $\rho_i$. Since the objectives are separable across clients $i \in \mathcal{I}'$, $J$ is quasi-concave in $\boldsymbol{\rho}$ jointly.

\textbf{(iv) Continuity.} $w_i(\rho_i, \lambda)$ is continuous in both arguments on the interior of $\Lambda \times \mathcal{P}$ (it is a rational function with positive denominator). Thus $J$ is jointly continuous.

By Sion's minimax theorem, $\min_\lambda \max_{\boldsymbol{\rho}} J =
\max_{\boldsymbol{\rho}} \min_\lambda J$, and there exists a saddle point $(\lambda^*, \boldsymbol{\rho}^*)$ satisfying:
\begin{equation}
  J(\lambda^*, \boldsymbol{\rho}) \;\le\; J(\lambda^*, \boldsymbol{\rho}^*)
  \;\le\; J(\lambda, \boldsymbol{\rho}^*)
  \quad \forall\, \lambda \in \Lambda,\; \boldsymbol{\rho} \in \mathcal{P}
  \label{eq:saddle}
\end{equation}
\end{proof}

\noindent\textbf{Interpretation.}
At equilibrium, $\lambda^* = \lambda_{\max}$ forces each adversary's best response to $\rho_i^* = (R_i+\theta)/2\lambda_{\max} < 1$, deterring full poisoning and confirming the low attack success rates observed in Figure~\ref{fig:backdoor_grid} even at 30\% malicious participation.

\section{FedBBA: Federated Backdoor and Behavior Analysis}
\label{fedbba}

In this section, we present the operational flow of our proposed federated learning defense framework, \textit{FedBBA} (Federated Backdoor and Behavior Analysis).

FedBBA enhances the robustness of federated learning by identifying and mitigating backdoor attacks without accessing raw client data, thereby preserving data privacy. Instead, it analyzes only model updates using Projection Pursuit Analysis (PPA) and a dynamic reputation scoring system.

\begin{algorithm}
  \caption{Round-Based Execution of the FedBBA Framework}
  \label{alg:nnn}
  \begin{algorithmic}[1]
    \STATE Initialize global model $\mathcal{M}_0$, reputation scores $R_i$
    \FOR {each federated learning round $r$ in $1$ to $T$}
        \FOR {each client $i$ in $1$ to $N$}
            \STATE Train local model $\mathcal{M}_i$ on client's data
            \STATE Send local model $\mathcal{M}_i$ to server
        \ENDFOR
        \STATE Calculate PPA score $s_i = \mathrm{PPA}(\mathcal{M}_i)$
        \STATE Cluster PPA scores using DBSCAN to determine classes
        \FOR {each client $i$ in classes}
            \STATE Compute normalized weights using Equation~(\ref{eq:weight_joint})
            \IF {$\text{class}_i == 0$}
                \STATE Update reputation score $R_i$ negatively using Equation~(\ref{eq:pena})
            \ELSE
                \STATE Update reputation score $R_i$ positively using Equation~(\ref{eq:reward})
            \ENDIF
        \ENDFOR
        \STATE Update global model using Equation~(\ref{eq:agg})
        \STATE Calculate server's payoff using Equation~(\ref{eq:server_utility})
        \STATE Calculate attacker's payoff using Equation~(\ref{eq:adv_utility})
    \ENDFOR
  \end{algorithmic}
\end{algorithm}

The detailed step-by-step procedure is formalized in Algorithm~\ref{alg:nnn}. At the start of training ($t = 0$), the server initializes the global model $\mathcal{M}$ and assigns an initial reputation score $R_i$ for each client. These scores are based on both historical behavior and gradient variation, as defined in Equation~(11) (Line 1). During each federated learning round $r$ from $1$ to $T$, selected clients train local models $\mathcal{M}_i$ on their private data and transmit them back to the server (Lines 2--5). The server then computes PPA scores $s_i$ to quantify the level of anomalous behavior in each update (Line 7), followed by clustering using DBSCAN to identify suspicious clients (Line 8). Based on these clusters, client reputation scores $R_i$ are updated: suspicious clients are penalized (Lines 11--12) using Equation~\eqref{eq:pena}, while benign clients are rewarded (Lines 13--14) via Equation~\eqref{eq:reward}. Normalized aggregation weights are computed using both reputation and PPA scores (Line 10), and the global model is updated accordingly using Equation~\eqref{eq:agg} (Line 16). Finally, the server and attacker payoffs are calculated using the MiniMax game-based utility functions (Equations~\ref{eq:server_utility} and~\ref{eq:adv_utility}, Lines 18--19). This iterative process allows FedBBA to dynamically detect and suppress malicious behaviors across rounds.

Figure~\ref{fig:game_workflow} illustrates the strategic execution of the proposed Minimax game framework. The interaction starts with the Attack Leader solving the maximization problem to derive the constrained optimal poisoning ratio $\rho_i^* = \min\{(R_i + \theta)/2\lambda,\, 1\}$ via Equation~\eqref{eq:rho_star}, which balances attack impact with detectability given the server's chosen defense sensitivity $\lambda$ (Step~1). Following this, both malicious and benign clients perform local training and transmit their model updates $\mathcal{M}_i$ to the server (Step~2). Upon receiving the updates, the Federated Server initiates the defense phase by calculating Projection Pursuit Analysis (PPA) scores and utilizing DBSCAN to cluster the updates and identify anomalies (Step~3). Based on these clusters, the server computes the aggregation weights $w_i(\rho_i, \lambda^*)$ for all clients using Equation~\eqref{eq:weight_joint}, suppressing the influence of high-poisoning clients consistent with the equilibrium defense sensitivity $\lambda^*$ (Step~4). The server then aggregates the weighted local models to produce the global model update $\mathcal{M}^{(t+1)}$ via Equation~\eqref{eq:agg} (Step~5). Finally, the feedback loop is closed as the server updates the reputation scores $R_i$ penalizing identified malicious actors via Equation~\eqref{eq:pena} and rewarding honest participants via Equation~\eqref{eq:reward}, which reduces the adversary's optimal poisoning ratio $\rho_i^*$ in the subsequent round via Equation~\eqref{eq:rho_star}, as a lower $R_i$ directly decreases the equilibrium best response (Step~6).

\section{Simulation \& Experiments}
\label{exp}
In this section, we present the experimental results demonstrating the effectiveness of our approach in minimizing the effects of backdoor attacks while maintaining the accuracy of the global model.

\subsection{Context \& Environment}

We conduct our experiments using the GTSRB and BTSC datasets, which contain 43 and 62 traffic sign classes respectively, to evaluate classification performance. The implementation relies on industry-standard libraries including Torch, Torchvision, NumPy, and Pandas.

We employ a convolutional neural network (CNN) composed of convolutional layers for feature extraction and two fully connected layers for classification. The poisoning ratio $\rho_i$ is dynamically determined by the adversary according to the optimization formulation described in Section~\ref{mg}. For the GTSRB dataset, we reduce it to 10 classes by selecting those with the highest number of samples. Each client is assigned 4 to 6 classes, with 250 to 350 samples per class. For the BTSC dataset, each client receives between 20 and 40 classes, with 50 to 150 images per class. This class-based partitioning results in a non-IID data distribution across clients, simulating realistic federated settings where participants hold heterogeneous local datasets. The federated learning setting consists of a total of $C=300$ clients, from which $c=40$ clients are randomly selected in each training round. Among the participating clients, up to $30\%$ are malicious.

We use two primary metrics to evaluate our approach:
\begin{itemize}
    \item \textbf{Accuracy}: the percentage of correctly classified images.
    \item \textbf{Backdoor attack success rate}: the percentage of triggered (backdoored) test samples misclassified by the global model into the attacker’s predefined target label.
\end{itemize}

We compare our approach (\textit{FedBBA}) with two state-of-the-art methods and a baseline:
\begin{itemize}
    \item \textbf{Vanilla}: standard federated learning with random client selection, used as a baseline \cite{mcmahan2017communication}.
    \item \textbf{RoPE}: a state-of-the-art method that applies PCA for backdoor detection \cite{wang2024rope}.
    \item \textbf{RDFL}: a state-of-the-art approach using clustering and noise clipping for robust aggregation \cite{wang2023adaptive}.
\end{itemize}

For the clustering component of our approach, we use the DBSCAN algorithm. The parameters are tuned through grid search, with the neighborhood radius set to \(\varepsilon = 0.6\) and the minimum samples parameter set to 5. These values balance cluster cohesion and noise filtering, ensuring stable clustering results across the datasets.

\begin{figure*}[!t]
    \centering
    \includegraphics[width=0.90\textwidth]{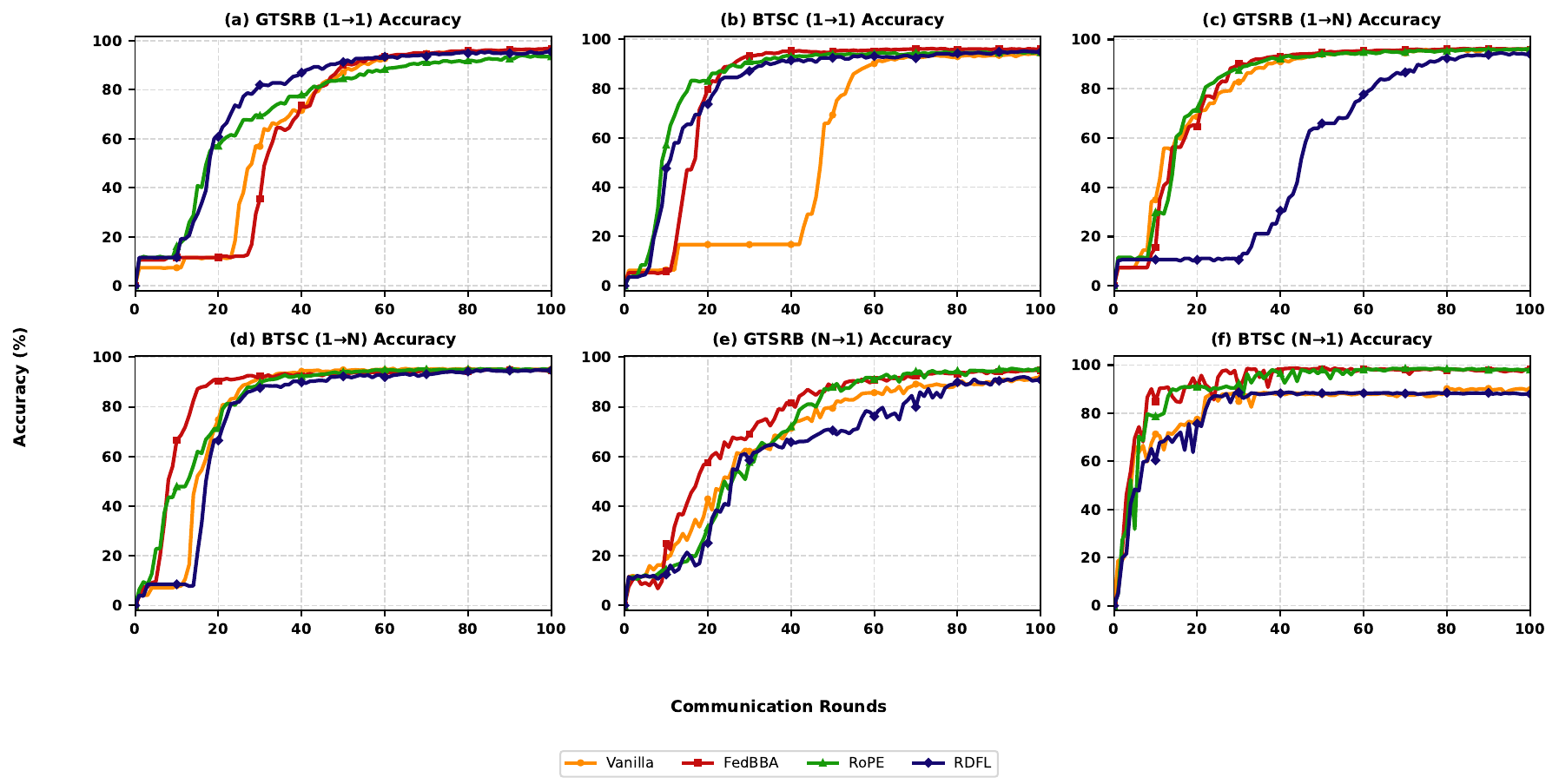}
    \caption{Normal task accuracy comparison of RDFL, Vanilla, FedBBA, and RoPE across datasets and attack settings over 100 communication rounds.}
    \label{fig:accuracy_grid}
\end{figure*}

\begin{figure*}[!t]
    \centering
    \includegraphics[width=0.90\textwidth]{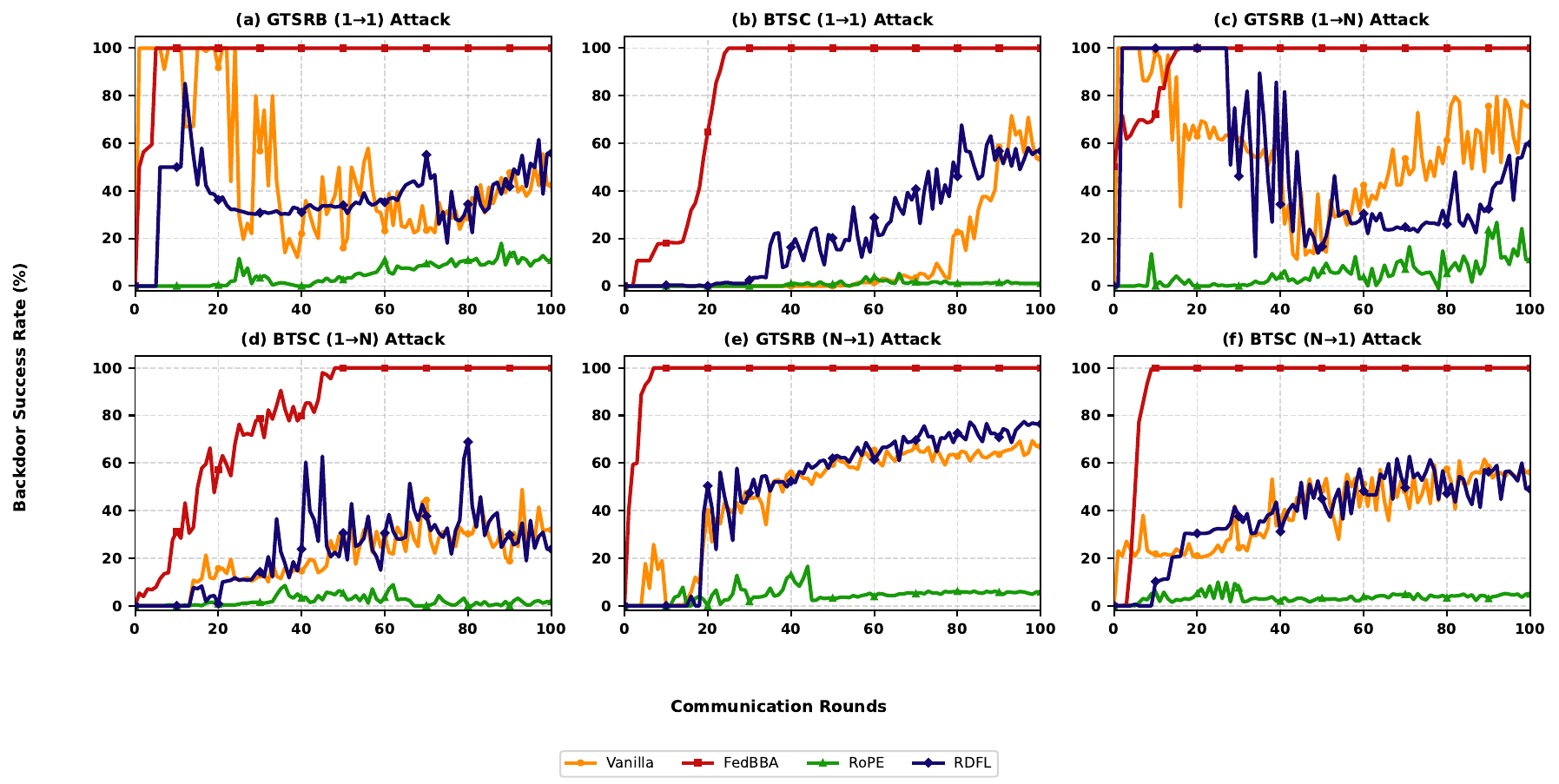}
    \caption{Backdoor success rate comparison of RDFL, Vanilla, FedBBA, and RoPE across datasets and attack scenarios over 100 communication rounds.}
    \label{fig:backdoor_grid}
\end{figure*}

\subsection{Experimental Results}
\label{er}
In this section, we present the results of our experiments in terms of accuracy and backdoor attack success rate, providing a detailed comparative analysis with baseline and state-of-the-art approaches.

To evaluate the impact of the reputation score's weighting parameters, we performed a sensitivity analysis on the values of $\alpha$ and $\beta$, which represent the contribution of historical behavior and gradient-based anomaly signals, respectively. We constrained $\alpha + \beta = 1$ and varied each to assess the effect on both clean-task accuracy and backdoor success rate. The experiments were conducted on the GTSRB dataset with a training duration of 100 federated learning rounds.

\begin{table}[ht]
  \centering
  \caption{Impact of $\alpha$ and $\beta$ on Model Accuracy and Backdoor Success Rate using GTSRB dataset over 100 rounds}
  \label{tab:alpha_beta_sensitivity}
  \small
  \begin{tabular}{cccc}
    \toprule
    $\alpha$ & $\beta$ & Accuracy (\%) & Backdoor Success Rate (\%) \\
    \midrule
    0.1 & 0.9 & 96.6 & 25.0 \\
    0.2 & 0.8 & 95.4 & 28.0 \\
    0.3 & 0.7 & 96.9 & 21.0 \\
    0.4 & 0.6 & 96.2 & 33.0 \\
    \textbf{0.5} & \textbf{0.5} & \textbf{96.8} & \textbf{16.0} \\
    0.6 & 0.4 & 96.2 & 22.0 \\
    0.7 & 0.3 & 95.6 & 41.0 \\
    0.8 & 0.2 & 95.5 & 21.0 \\
    0.9 & 0.1 & 96.1 & 45.0 \\
    \bottomrule
  \end{tabular}
\end{table}

As illustrated in Table~\ref{tab:alpha_beta_sensitivity}, the balanced configuration $\alpha = \beta = 0.5$ provides a better trade-off between clean accuracy and backdoor resilience. It achieves one of the highest accuracy values 96.8\%, while minimizing the backdoor success rate to just 16\%. In contrast, skewed configurations $\alpha = 0.9$, $\beta = 0.1$ result in a significantly higher backdoor success rate up to 45\%, indicating reduced robustness. This empirical evidence justifies our choice of equal weights in the reputation calculation. However, we recognize that optimal values may vary across applications, and tuning may be warranted in future work depending on domain-specific factors and threat severity.

\subsubsection{One-to-One Attacks}

Figure~\ref{fig:accuracy_grid}(a) illustrates the performance of our \textit{FedBBA} approach in mitigating backdoor attacks while maintaining high normal task accuracy using the GTSRB dataset. Initially, all methods exhibit nearly similar accuracy levels; however, as training progresses, \textit{FedBBA}'s accuracy significantly surpasses that of the other approaches. By the 20th round, \textit{FedBBA} achieves an accuracy exceeding 60\%, which is markedly higher than RDFL and RoPE, whose accuracy remains below 15\%. This trend continues with \textit{FedBBA} reaching an accuracy of approximately 96.8\% by the 100th round, compared to 93.4\% for RDFL, 95.5\% for Vanilla, and 91.5\% for RoPE. While the improvement in clean-task accuracy over baseline methods may appear modest $\sim$3.4--11\%, it is statistically meaningful and, more importantly, shows that the mitigation happened without hindering the performance on normal tasks. This is especially important in the context of backdoor attacks, which are stealthy and preserve normal task accuracy to evade detection. Many existing defenses apply techniques such as excluding clients or pruning updates, leading to significant degradation in normal task performance. In contrast, FedBBA achieves a strong balance by preserving high accuracy on the normal task while substantially reducing the backdoor success rate to the lowest $\sim$1--10\%. These results underscore \textit{FedBBA's} robust capability in maintaining model integrity and performance, effectively addressing backdoor attacks more efficiently than the other methods evaluated.

Figure~\ref{fig:backdoor_grid}(a) presents the backdoor attack success rate results, highlighting the effectiveness of the \textit{FedBBA} technique using the GTSRB dataset. The \textit{Vanilla} baseline confirms the efficacy of the attack, rapidly reaching and maintaining a 100\% success rate, thereby establishing the severity of the threat in an unprotected setting. In sharp contrast, \textit{FedBBA} effectively suppresses the backdoor, maintaining a complete 0\% attack success rate for the first 18 rounds. While other defense methods like RDFL and RoPE exhibit significant instability, with RoPE peaking at 85.2\% and RDFL fluctuating considerably, \textit{FedBBA} consistently keeps the attack impact minimal. By the 100th round, \textit{FedBBA} restricts the attack success rate to just 10.8\%, significantly outperforming RDFL (42.4\%) and RoPE (55.8\%).

Figure~\ref{fig:accuracy_grid}(b) shows the performance of \textit{FedBBA} in maintaining high normal task accuracy compared to RDFL, Vanilla, and RoPE using the BTSC dataset. Initially, all methods exhibit similar accuracy levels, with \textit{FedBBA} starting at 3.4\% while RDFL, Vanilla, and RoPE start at 6.2\%, 5.3\%, and 3.5\%, respectively. As training progresses, \textit{FedBBA}'s accuracy significantly surpasses that of the other approaches in the early stages. By the 10th round, \textit{FedBBA} achieves an accuracy of 57.1\%, markedly higher than RDFL (6.4\%), Vanilla (5.9\%), and RoPE (47.6\%). By the 20th round, \textit{FedBBA} achieves an accuracy of 83.0\%, while RDFL, Vanilla, and RoPE reach 16.7\%, 79.6\%, and 73.6\%, respectively. This trend stabilizes with \textit{FedBBA} reaching an accuracy of approximately 94.6\% by the 100th round. In comparison, RDFL, Vanilla, and RoPE exhibit accuracies of 94.2\%, 95.9\%, and 94.8\%, respectively.

Figure~\ref{fig:backdoor_grid}(b) highlights the effectiveness of the \textit{FedBBA} approach in mitigating backdoor attacks compared to RDFL, Vanilla, and RoPE using the BTSC dataset. Initially, no approach exhibits vulnerability. This delay can be attributed to the complexity and breadth of the BTSC dataset, which necessitates a higher number of training rounds for the backdoor mechanism to establish persistence. This behavior is clearly observed in the Vanilla approach, where, even in the absence of any defense mechanism, the attack required more than 20 rounds to achieve a 100\% success rate.

As training progresses beyond this initial phase, \textit{FedBBA} demonstrates remarkable resilience compared to the baselines. For example, at round 38, \textit{FedBBA} and RDFL still show an attack success rate of 0\%, whereas Vanilla has already stabilized at 100\% and RoPE has risen to approximately 7.8\%. By round 100, \textit{FedBBA} successfully suppresses the attack success rate to just 1.1\%, significantly lower than RDFL (53.5\%), RoPE (56.7\%), and Vanilla (100\%). This consistent suppression underscores \textit{FedBBA}'s superior robustness in mitigating backdoor attacks throughout the training process.

\subsubsection{One-to-N Attacks}

Figure~\ref{fig:accuracy_grid}(c) presents the main task accuracy of the global model throughout the training process. Despite the presence of the attack, \textit{FedBBA} demonstrates a high level of utility preservation, achieving a final accuracy of 95.38\% by round 100. This performance is comparable to, and slightly exceeds, the Vanilla baseline, which reaches 94.69\%. In contrast, the other defense mechanisms show a notable degradation in model utility, with RDFL and RoPE achieving lower final accuracies of 91.78\% and 90.81\%, respectively. This indicates that \textit{FedBBA} successfully defends against the attack without compromising the model's ability to learn the primary task.

Figure~\ref{fig:backdoor_grid}(c) highlights the effectiveness of the approaches in mitigating the backdoor attack. The Vanilla model is compromised almost immediately, with the attack success rate jumping to 40\% at round 1 and reaching total saturation of 100\% as early as round 7, where it remains for the duration of training. RoPE and RDFL also fail to prevent the attack, ending with high attack success rates of 76.4\% and 67\%, respectively. Conversely, \textit{FedBBA} exhibits superior robustness, maintaining the attack success rate near zero for most of the training process. By round 100, \textit{FedBBA} suppresses the attack success rate to just 5.84\%, significantly outperforming all other methods and effectively neutralizing the backdoor threat.

Figure~\ref{fig:accuracy_grid}(d) presents the main task accuracy of the global model throughout the training process. \textit{FedBBA} demonstrates exceptional utility preservation, achieving a final accuracy of 98.06\% by round 100. This performance is effectively on par with the Vanilla baseline, which reaches 98.30\%, indicating that the defense mechanism does not hinder the model's ability to learn the complex features of the BTSC dataset. In contrast, the other defense mechanisms suffer from significant utility loss, with RDFL and RoPE dropping to final accuracies of 89.79\% and 87.98\%, respectively.

Figure~\ref{fig:backdoor_grid}(d) highlights the effectiveness of the approaches in mitigating the backdoor attack. The Vanilla model exhibits extreme vulnerability, with the attack success rate surging rapidly to 99.6\% by round 9 and reaching total saturation (100\%) by round 10, where it remains fixed. RoPE and RDFL also fail to provide adequate protection, ending the training with high attack success rates of 49.02\% and 56.10\%, respectively. Conversely, \textit{FedBBA} shows superior resilience. Despite the aggressive nature of the attack, \textit{FedBBA} consistently suppresses the backdoor, maintaining a low attack success rate throughout the session and concluding at just 4.88\% by round 100. This sharp contrast with the baselines underscores \textit{FedBBA}'s capability to neutralize One-to-N backdoor injections on the BTSC dataset.

\subsubsection{N-to-One Attacks}

Figure~\ref{fig:accuracy_grid}(e) evaluates the performance of various defense mechanisms under a multiple-to-one backdoor attack using the GTSRB dataset. Initially, at round 1, all defenses exhibit similar performance with accuracy rates around 7.41\% for RDFL and Vanilla, 11.57\% for \textit{FedBBA}, and 10.19\% for RoPE. As training progresses, \textit{FedBBA} consistently outperforms or matches the other methods. By round 20, \textit{FedBBA}’s accuracy escalates to 71.74\%, higher than RDFL (68.66\%), Vanilla (64.55\%), and significantly outperforming RoPE (10.56\%). By round 100, \textit{FedBBA} reaches a peak accuracy of 96.14\%, effectively matching or slightly exceeding RDFL (96.06\%) and Vanilla (96.02\%), while remaining superior to RoPE (93.97\%).

Figure~\ref{fig:backdoor_grid}(e) underscores the superior efficacy of the \textit{FedBBA} approach in comparison to RDFL and RoPE in reducing the backdoor success rate under the multi-trigger setting using the GTSRB dataset. Initially, Vanilla exhibits immediate vulnerability, starting at 50.2\% in round 0, while RDFL spikes to 100\% by round 1. In contrast, \textit{FedBBA} starts with an attack success rate of 0\% and, despite the aggressive attack environment, keeps the rate significantly lower than the other methods throughout the session. For example, at round 40, \textit{FedBBA} shows an attack success rate of just 4.2\%, whereas Vanilla has already saturated at 100\%, RDFL is at 54.6\%, and RoPE is at 34.4\%. By round 100, \textit{FedBBA} suppresses the attack success rate to 11.0\%, which is significantly lower than RDFL (75.6\%), RoPE (59.8\%), and Vanilla (100\%).

Figure~\ref{fig:accuracy_grid}(f) illustrates the performance of various defense mechanisms in the normal task accuracy under a multiple-to-one backdoor attack scenario using the BTSC dataset. Initially, all methods exhibit comparable accuracy levels below 7\%. By round 10, Vanilla takes an early lead with 66.5\%, while \textit{FedBBA} achieves 47.9\%, RDFL 8.1\%, and RoPE 8.5\%. However, \textit{FedBBA} demonstrates steady and robust learning as training progresses. By round 100, \textit{FedBBA} reaches a peak accuracy of 95.08\%, effectively matching RDFL (95.16\%) and outperforming both Vanilla (94.44\%) and RoPE (94.72\%). This confirms that \textit{FedBBA} maintains high utility even under complex attack scenarios.

Figure~\ref{fig:backdoor_grid}(f) highlights the effectiveness of \textit{FedBBA} in mitigating backdoor attacks compared to RDFL and RoPE in a multiple-to-one attack scenario using the BTSC dataset. \textit{FedBBA} demonstrates a significant improvement, consistently achieving lower success rates. Specifically, \textit{FedBBA} starts with a success rate of 0\% and, despite the attack's persistence, keeps the rate significantly low throughout the rounds. For instance, by round 50, \textit{FedBBA} suppresses the attack success rate to 5.3\%, whereas Vanilla has already reached total saturation (100\%), and RDFL and RoPE have risen to 25.6\% and 30.6\%, respectively. By round 100, \textit{FedBBA} further limits the attack success rate to 1.77\%, significantly lower than RDFL (31.9\%), RoPE (23.8\%), and Vanilla (100\%).

In summary, the results illustrate that \textit{FedBBA} consistently outperforms RDFL and RoPE in maintaining high accuracy and mitigating backdoor attacks. Across different attack scenarios and datasets, \textit{FedBBA} proves to be a robust defense mechanism, specifically in the N-to-One and One-to-N coordinated backdoor attack, which reflects the framework's ability to detect the pattern of the malicious data in the suspicious models, effectively enhancing model integrity and performance while addressing the challenges posed by backdoor attacks.

\section{Conclusion}
\label{conc}

In this work, we proposed a novel and comprehensive backdoor attack detection and mitigation scheme for federated learning called \textit{FedBBA}. Our framework uniquely integrates Projection Pursuit Analysis (PPA), a robust reputation system, and game-theoretic incentives to neutralize malicious clients. Extensive simulations on the GTSRB and BTSC datasets demonstrate that \textit{FedBBA} significantly outperforms existing state-of-the-art methods, specifically RDFL and RoPE, across diverse attack scenarios, including One-to-N and N-to-One attacks. Empirical results confirm that \textit{FedBBA} effectively suppresses backdoor attack success rates to between 1.1\% and 11\%, whereas baseline defenses often failed to contain the threat, yielding success rates as high as 76\%. Crucially, this robustness is achieved without hindering the global model's performance; \textit{FedBBA} consistently maintained normal task accuracy in the range of 95\% to 98\%. While \textit{FedBBA} demonstrates strong effectiveness, its evaluation is restricted to image-based classification tasks. Future work will explore generalizing the framework to diverse modalities, such as natural language processing and IoT sensor data. Additionally, we plan to incorporate adaptive parameter learning to dynamically tune the reputation weights ($\alpha$ and $\beta$), further enhancing the framework's scalability and responsiveness to evolving threat landscapes.

\bibliographystyle{IEEEtran}
\bibliography{fedbba}

\vskip 0pt plus -1fil
\vspace{-33pt}

\begin{IEEEbiographynophoto}{Osama Wehbi}
is a Ph.D. candidate in the Department of Computer and Software Engineering at Polytechnique Montreal, Canada. His primary research areas include cybersecurity, federated learning, and game theory.
\end{IEEEbiographynophoto}

\vskip 0pt plus -1fil
\vspace{-33pt}

\begin{IEEEbiographynophoto}{Sarhad Arisdakessian}
is a Ph.D. candidate in the Department of Computer and Software Engineering at Polytechnique Montreal, Canada. His research interests lie in federated learning and game theory.
\end{IEEEbiographynophoto}

\vskip 0pt plus -1fil
\vspace{-33pt}

\begin{IEEEbiographynophoto}{Omar Abdel Wahab}
holds an assistant professor position with the Department of Computer and Software Engineering, Polytechnique Montreal, Canada. His research activities focus on cybersecurity, Internet of Things, and artificial intelligence.
\end{IEEEbiographynophoto}

\vskip 0pt plus -1fil
\vspace{-33pt}

\begin{IEEEbiographynophoto}{Anderson R. Avila}
is a Ph.D. candidate at the Institut National de la Recherche Scientifique (INRS), Canada. His research interests include speaker and emotion recognition, pattern recognition, and multimodal signal processing for biometric applications.
\end{IEEEbiographynophoto}

\vskip 0pt plus -1fil
\vspace{-33pt}

\begin{IEEEbiographynophoto}{Azzam Mourad}
is a Professor of Computer Science with Khalifa University, UAE. His research interests include cloud computing, artificial intelligence, and cybersecurity.
\end{IEEEbiographynophoto}

\vskip 0pt plus -1fil
\vspace{-33pt}

\begin{IEEEbiographynophoto}{Hadi Otrok}
is a Professor and Chair of the Department of Electrical Engineering and Computer Science at Khalifa University, UAE. His research interests include computer and network security, crowd sensing and sourcing, ad hoc networks, and cloud security.
\end{IEEEbiographynophoto}

\end{document}